\documentclass{amsart}

\usepackage{graphicx,amssymb,latexsym,amsfonts,mathrsfs,amsmath,xcolor}
\usepackage{enumerate}       
 \definecolor{darkgreen}{rgb}{0,0.5,0}
 \newcommand{\Aquad}{\subscr{\mathscr{A}}{quad}}
 \newcommand{\sout}[1]{}

\usepackage[T1]{fontenc}    
\usepackage{hyperref}       
\usepackage{url}            
\usepackage{booktabs}       
\usepackage{amsfonts}       
\usepackage{nicefrac}       
\usepackage{microtype}      
\usepackage{amssymb,latexsym,amsfonts,amsmath}
\usepackage{graphicx}
\usepackage{color}
\usepackage{mathabx}
\usepackage{cite}

\newtheorem{theorem}{Theorem}[section]

\newtheorem{lemma}[theorem]{Lemma}

\newtheorem{problem}[theorem]{Problem}

\newtheorem{proposition}[theorem]{Proposition}
\newtheorem{corollary}[theorem]{Corollary}

\newtheorem{definition}[theorem]{Definition}

\newtheorem{remark}[theorem]{Remark}
\numberwithin{equation}{section}

\newcommand{\R}{{\mathbb{R}}}

\newcommand{\N}{{\mathbb{N}}}

\renewcommand{\dim}{\mathrm{dim}}

\newcommand{\cl}{\mathrm{cl}}

\newcommand{\real}{{\mathbb{R}}}

\newcommand{\vect}{{\mathrm{vec}}}

\newcommand{\pder}[2]{\frac{\partial #1}{\partial #2}}

\newcommand\subscr[2]{#1_{\textup{#2}}}
\newcommand\upscr[2]{#1^{\textup{#2}}}

\newcommand{\xin}{\upscr{x}{init}}
\newcommand{\xf}{\upscr{x}{fin}}
\newcommand{\lin}{\upscr{\ell}{init}}
\newcommand{\lf}{\upscr{\ell}{fin}}
\newcommand{\Xin}{\upscr{X}{init}}
\newcommand{\tin}{\upscr{t}{init}}
\newcommand{\Xf}{\upscr{X}{fin}}
\newcommand{\tf}{\upscr{t}{fin}}
\newcommand{\Ain}{\upscr{A}{init}}
\newcommand{\Af}{\upscr{A}{fin}}

\newcommand{\sampleset}{\subscr{E}{samples}}
\newcommand{\VecE}{Z}

\newcommand{\oprocendsymbol}{\hbox{$\bullet$}}
\newcommand{\oprocend}{\relax\ifmmode\else\unskip\hfill\fi\oprocendsymbol}

\newcommand{\longthmtitle}[1]{\mbox{}\textup{\textbf{(#1):}}}

\begin{document}

\begin{abstract}
In this paper we show that deep residual neural networks have the power of universal approximation by using, in an essential manner, the observation that these networks can be modeled as nonlinear control systems. We first study the problem of using a deep residual neural network to exactly memorize training data by formulating it as a controllability problem for an ensemble control system. Using techniques from geometric control theory, we identify a class of activation functions that 
allow us to ensure controllability on an open and dense submanifold of sample points. Using this result, and resorting to the notion of 
monotonicity, we establish that any continuous function can be approximated on a compact set to arbitrary accuracy, with respect to the uniform norm, by this class of neural networks. Moreover, we provide optimal bounds on the number of required neurons. 
\end{abstract}

\title[]{Universal approximation power of deep neural networks\\via nonlinear control theory}
\thanks{The work of the first author was supported the CONIX research center, one of six centers in JUMP, a Semiconductor Research Corporation (SRC) program sponsored by DARPA. The work of the second author was supported by the Alexander von Humboldt Foundation, and the Natural Sciences and Engineering Research Council of Canada.}

\author[Paulo Tabuada]{Paulo Tabuada}
\address{Department of Electrical and Computer Engineering\\
University of California at Los Angeles,
Los Angeles, CA 90095}
\email{tabuada@ee.ucla.edu}
\urladdr{http://www.ee.ucla.edu/$\sim$tabuada}

\author[Bahman Gharesifard ]{Bahman Gharesifard }
\address{}
\email{gharesifard@ucla.ca}
\urladdr{https://gharesifard.github.io}

\maketitle

\section{Introduction} 

The use of control-theoretic techniques to analyze the input-output behavior of neural networks has a long history within the control community. The classical work of Cybenko (see~\cite{GC:89}) on universal approximation that appeared in the Mathematics of Control, Signals, and Systems is a testimony to the role that system theory played in advancing this subject from early on. Prior to this, the IEEE Control System Magazine had a special section on neural network for systems and control which addressed the stability~\cite{BB:88} and training~\cite{DP-AS-AAY:88} of neural networks. In fact, the use of dynamical and control systems to describe and analyze neural networks goes back at least to the 70's. For example, Wilson-Cowan's  equations~\cite{W+C} are differential equations and so is the model proposed by Hopfield in~\cite{Hopfield3088}. These techniques have been used to study several problems such as weight identifiability from data~\cite{ALBERTINI1993975,albertini1993uniqueness}, controllability~\cite{SONTAG1999121,sontag1997complete}, stability~\cite{20200,HIRSCH1989331}, and approximation of dynamical systems~\cite{KF-YN:93}.

Using neural networks within a control-loop was immediately well-motivated and, had it not been due to training issues in large-scale settings, the subject would have developed even faster~\cite{PJA:90}. Equally important to this discussion is the lack of robustness guarantees when placing a neural network in closed-loop system, which we will not dwell on here, as the primary subject of our work is exploring the universal approximation power of deep neural network. 


Roughly speaking, the universal approximation problem for neural networks asks if a neural network can be used to approximate an unknown function when the only knowledge available about such function is its evaluation on a finite number of sample points. The quality of the approximation is measured using a norm and different norms lead to different answers to this problem. In this paper we focus on the uniform norm, noting that uniform approximation results on compact sets in $L^p$ norms follow as a special case. 
Much of earlier work at the intersection of control theory and neural networks focused on neural networks with \emph{unbounded width}, i.e., with no bound on the number of nodes that can be used in each layer of the architecture of the neural network. Having no bound on the width is essential in the classical work on universal approximation, see for example~\cite{KH:91,AP:99,GC:89}. In contrast, \emph{deep} neural networks have many layers, but bounded width. It has been empirically observed that deep networks have better approximation capabilities than their shallow counterparts and are easier to train~\cite{10.5555/2969033.2969123,DBLP:conf/iclr/UrbanGKAWMPRC17}. An intuitive explanation for this fact is based on the different ways in which these types of networks perform function approximation. While shallow networks prioritize parallel compositions of simple functions (the number of neurons per layer is a measure of parallelism), deep networks prioritize sequential compositions of simple functions (the number of layers is a measure of sequentiality).

\subsection{Related work}
The objective of this paper is to demonstrate how tools from control theory can be used to provide sharp universal approximation results for deep residual neural networks. Earlier work on controllability of differential equation models for neural networks, e.g.,~\cite{SONTAG1999121}, assumed the weights to be constant and that an exogenous control signal was fed into the neurons. In contrast, we regard the weights as control inputs and that no additional exogenous control input is present. These two different interpretations of the model lead to two very different technical problems. To elaborate further, we consider controllability problems where the same input signal (the weights) is used to simultaneously steer multiple initial states (corresponding to sample points in the domain of the function to be approximated) to multiple final states (corresponding to the evaluation of the function to be approximated on the sample points). This naturally leads to the problem of controlling ensembles of control systems. Most of the work in this area, see for instance~\cite{LJS-KN:06,UH-MS:14,brockett2007optimal}, considers parametrized ensembles of control systems, i.e., families of control systems where each member of the family is different. In the setting of our work, the ensemble control system is formed by exact copies of the control system modeling the neural network, albeit initialized at different sample points. In this sense, our work is most closely related to the setting of~\cite{AA-AS:20,AA-AS:20-deep} where controllability results for ensembles of infinitely copies of the same control system are provided. Whereas the results in~\cite{AA-AS:20,AA-AS:20-deep} extend Lie algebraic techniques for controllability from finite dimensional to infinite dimensional state spaces, we use Lie algebraic techniques to study controllability of finite ensembles and obtain approximation results for infinite ensembles by using the notion of monotonicity. Moreover, by focusing on the specific control systems arising from deep residual networks we are able to provide easier to verify controllability conditions than those provided in~\cite{AA-AS:20,AA-AS:20-deep} for more general control systems. Controllability of finite ensembles of control systems motivated by neural network applications was investigated in~\cite{CC-ML-JM:19} where it is shown that controllability is a generic property and that, for control systems that are linear in the inputs, five inputs suffice. These results are insightful but they do not apply to specific control systems such as those describing residual networks and studied in this paper. Moreover the results in~\cite{CC-ML-JM:19} do not address the problem of universal approximation in the infinity norm. Finally, it is worth pointing out to the work~\cite{CE-BG-DP-EZ:20} where a neural ordinary differential equation perspective on supervised learning is provided inspired by turnpike theory in optimal control.

In the machine learning community, dynamical and control systems have been utilized  for the analysis of deep neural networks~\cite{Proposal,Haber_2017,lu2017finite}. The  universal approximation problem has also been the subject of previous studies. In~\cite{OneNeuron}, by focusing on ReLU activation functions, it was shown that one neuron per layer suffices to approximate an arbitrary Lebesgue integrable real-valued function with respect to the $L^1$ norm. 
Closer to this paper are the results in~\cite{QL-TL-ZS:19} establishing universal approximation,  with respect to the $L^p$ norm, $1\le p<\infty$,  based on a general sufficient condition satisfied by several examples of activation functions. In contrast, in this paper we establish universal approximation in the stronger sense of the uniform norm $L^\infty$. 
Universal approximation with respect to the infinity norm for non-residual deep networks, allowing for  general classes of activation functions, was recently established in~\cite{PK-TL:20}. In particular, it is shown in~\cite{PK-TL:20} that under very mild conditions on the activation functions, any continuous function $ f: K \rightarrow \real^m $, where $ K \subset \real^n $ is compact, can be approximated in the infinity norm using a deep neural network of width $ n+m+2 $. Even though these results do not directly apply to residual networks (due to the presence of skip connections), they require $2n+2$ neurons for the case discussed in this paper where $n=m$. In contrast, one of our main results, Corollary~\ref{theorem:n+1}, asserts that a width of $2n+1$ is sufficient for universal approximation. \sout{Particularly, in light of the results of}~\cite{JJ:19}\sout{, this width is optimal for uniform approximation. }

To complete our literature review, we should point out the results in~\cite{park2021minimum} that were concurrently developed\footnote{A preliminary version of the results in this paper (see~\cite{PT-BG:21}) was presented at the same conference as~\cite{park2021minimum} and we learned of the work in~\cite{park2021minimum} through the list of accepted papers at this conference.} with ours. In~\cite{park2021minimum} it is established that $\max\{n+1,m\}$ neurons per layer suffice for universal approximation, although the results do not directly apply to residual networks.
\sout{When $n=m$, the case discussed in our work, both papers require $n+1$ neurons. However, there are several technical nuances across the results (for instance, some results in}~\cite{park2021minimum} \sout{require a layer of step activation functions, in addition to, e.g., ReLUs, whereas others hold even if the domain of $f$ is not compact when approximating in the $L^p$ norm) that deserve further investigation. In addition, the employed techniques are completely different, in particular, control theory is not used in}~\cite{park2021minimum}.

When $n=m$, the case discussed in our work, the results in~\cite{park2021minimum} state that $n+1$ neurons suffice. The larger number of neurons, $2n+1$, in our work is a consequence of approximating a given function by homeomorphisms rather than more general classes of maps. Nevertheless, we also identify conditions under which approximation is possible with $n$ neurons.

A shorter version of this paper was presented as~\cite{PT-BG:21} at the 2021 International Conference on Learning Representations. In addition to including complete proofs and further clarifications, we have made an attempt to provide a self-contained and pedagogical exposition accessible to the control community.

\subsection{Contributions}
We show in this paper that residual networks~\cite{ResNets} have universal approximation power when equipped with an activation function from a large class of functions. Specifically, given a finite set of points in the domain of the function to be learned, we cast the problem of designing weights for an approximating deep residual network as the problem of designing a single open-loop control input so that the solution of the control system modeling the neural network takes the finite set of points to its evaluation under the function to be approximated. In spite of the fact that we only have access to a single open-loop control input, we prove that the corresponding ensemble of control systems is controllable. Our results rely on the activation functions (or a suitable derivative) to satisfy a quadratic differential equation. Most activation functions in the literature either satisfy this condition or can be suitably approximated by functions satisfying it. We then utilize this controllability property to obtain universal approximability results for continuous functions in a uniform sense, i.e., with respect to the supremum norm. This is achieved by using the notion of monotonicity that lets us conclude uniform approximability on compact sets from controllability of finite ensembles. In doing so, we also establish that $2n+1$ neurons per layer suffice when the function to be approximated is defined of a compact subset of $\R^n$ and has codomain $\R^n$. A by-product of our work is a generalization-like bound that can be leveraged to obtain deterministic guarantee for the generalization error based on the training error.

\section{Control-theoretic view of residual networks}
\label{ControlView}
We start by providing a control system perspective on residual neural networks, following~\cite{Proposal,Haber_2017,lu2017finite}.
\subsection{From residual networks to control systems and back}
The update law for residual networks is given by:
\begin{equation}
\label{DiscreteTimeControlSystem}
x(k+1)=x(k)+S(k)\Sigma(W(k)x(k)+b(k)),
\end{equation}
where $k\in\N_0$ indexes each layer, $x(k)\in \R^n$ is the value of the $n$ neurons in layer $k$, and $(S(k),W(k),b(k))\in \R^{n\times n}\times \R^{n\times n}\times \R^n $ is the value of the weights for layer $k$. The quantities $x(k)\in \R^n$ and $(S(k),W(k),b(k))\in \R^{n\times n}\times \R^{n\times n}\times \R^n $ can be given the control theoretic interpretation of state and input when  when $k$ is viewed as indexing time. In~\eqref{DiscreteTimeControlSystem}, $S$, $W$, and $b$ are weight functions assigning weights to each time instant $k$, and $\Sigma:\R^n\to \R^n$ is of the form $\Sigma(x)=(\sigma(x_1),\sigma(x_2),\hdots,\sigma(x_n))$, where \mbox{$\sigma:\R^n\to \R^n$} is an \emph{activation function}. By drawing an analogy between~\eqref{DiscreteTimeControlSystem} and Euler's forward method to discretize differential equations, one can interpret~\eqref{DiscreteTimeControlSystem} as the time discretization of the continuous-time control system:
\begin{equation}
\label{ControlSystem}
\dot{x}(t)=S(t)\Sigma(W(t)x(t)+b(t)),
\end{equation}
where $ x(t) \in \R^n $ and $ (S(t),W(t),b(t))\in \R^{n\times n}\times \R^{n\times n}\times \R^n $; in what follows, and in order to make the presentation simpler, we sometimes drop the dependency on time. To make the connection between~\eqref{DiscreteTimeControlSystem} and~\eqref{ControlSystem} precise, let $x:[0,\tau]\to \R^n$ be a solution of the control system~\eqref{ControlSystem} for the control input $(S,W,b):[0,\tau]\to \R^{n\times n}\times \R^{n\times n}\times \R^n$, where $\tau \in \R^+$. Then, given any desired accuracy $\varepsilon\in\R^+$ and any norm $\vert \cdot\vert$ in $\R^n$, there exists a sufficiently small time step $T\in \R^+$ so that the function $z:\{0,1,\hdots, \lfloor \tau/T\rfloor\}\to \R^n$, with $ z(0)=x(0) $, defined by: 
\[
z(k+1)=z(k)+TS(kT)\Sigma(W(kT)z(k)+b(kT)),
\] 
approximates the sequence $\{x(kT)\}_{k=0,\hdots,\lfloor \tau/T\rfloor}$ with error $\varepsilon$, i.e.: 
$$\vert z(k)-x(kT)\vert\le \varepsilon,$$ for all $ k\in \{0,1,\hdots, \lfloor \tau/T\rfloor\} $. Intuitively, any statement about the solutions of~\eqref{ControlSystem} holds for the solutions of~\eqref{DiscreteTimeControlSystem} with arbitrarily small error $\varepsilon$, provided that we can choose the depth to be sufficiently large.

\subsection{Neural network training and controllability}
\label{NNTraining}
Given a function $f:\R^n \to \R^n$ and a finite set of samples $\sampleset\subset \R^n$, the problem of training a residual network so that it maps $x\in \sampleset$ to $f(x)$ can be phrased as the problem of constructing an open-loop control input $(S,W,b):[0,\tau]\to \R^{n\times n}\times \R^{n\times n}\times \R^n$ so that the resulting solution of~\eqref{ControlSystem} takes the states $x\in \sampleset$ to the states $f(x)$. The ability to approximate a function $f$ is tightly connected with a \emph{controllability problem}. This problem is, perhaps, less usual in that the same control input is to drive multiple initial states (the points in $\sampleset$) to multiple final states (the points in $f(\sampleset)$).
To make the controllability problem precise, it is convenient to consider the ensemble of $d=\vert \sampleset\vert$ copies of~\eqref{ControlSystem} given by the matrix differential equation:
\begin{equation}\label{Product} 
\begin{cases}
\dot{X}(t)=\left[ \dot{X}_{\bullet 1}(t)\vert \dot{X}_{\bullet 2}(t)\vert\hdots\vert \dot{X}_{\bullet d}(t) \right]&\\
\dot{X}_{\bullet 1}(t)=S(t)\Sigma(W(t)X_{\bullet 1}(t)+b(t))&\\
\dot{X}_{\bullet 2}(t)=S(t)\Sigma(W(t)X_{\bullet 2}(t)+b(t))&\\
\vdots&\\
\dot{X}_{\bullet d}(t)=S(t)\Sigma(W(t)X_{\bullet d}(t)+b(t)), &
\end{cases}
\end{equation}
where for  time $ t\in \R_0^+ $ the $i$th column of the matrix $X(t)\in \R^{n\times d}$, denoted by $X_{\bullet i}(t)$, is the solution of the $i$th copy of~\eqref{ControlSystem} in the ensemble. If we now  index the elements of $\sampleset$ as $\{x^1,\hdots,x^d\}$, where $ d $ is the cardinality of $ \sampleset$, and consider the matrices $\Xin=[x^1\vert x^2\vert\hdots\vert x^d]$ and $\Xf=[f(x^1)\vert f(x^2)\vert\hdots\vert f(x^d)]$, we see that the existence of a control input resulting in a solution of~\eqref{Product} starting at $\Xin$ and ending at $\Xf$, i.e., controllability of~\eqref{Product}, is equivalent to existence of an input for~\eqref{ControlSystem} so that the resulting solution starting at $x^i\in \sampleset$ ends at $f(x^i)$, for all $i\in \{1,\hdots,d\}$.

Note that achieving controllability of~\eqref{Product} is difficult, since all the copies of~\eqref{ControlSystem} in~\eqref{Product} are \emph{identical} and they all use the \emph{same input}. Therefore, to achieve controllability, we must have sufficient diversity in the initial conditions to overcome the symmetries present in~\eqref{Product}, see~\cite{CA-BG:14-acc}. Our controllability result, Theorem~\ref{Thm:Controllability}, describes precisely such diversity. As mentioned in the introduction, this observation also distinguishes the problem under study here from the classical setting of ensemble control~\cite{LJS-KN:06,UH-MS:14} where a collection of systems with \emph{different} dynamics are driven by the same control input, and instead is closely related to the recent work \cite{CC-ML-JM:19,AA-AS:20,AA-AS:20-deep}.

\section{Problem formulation}
Since the results in this paper already hold when $S(t)$ is a scalar multiple of the identity matrix, we directly simplify~\eqref{ControlSystem}, by replacing $S$ with the scalar-valued function $s$, that is:
\begin{equation}
\label{ControlSystemModel}
\dot{x}(t)=s(t)\Sigma(W(t)x(t)+b(t)),
\end{equation}
where $ x(t) \in \R^n $ and $ (s(t),W(t),b(t))\in \R\times \R^{n\times n}\times \R^n $.
In fact, we will later see
that it suffices to let $s$ assume two arbitrary values only (one positive and one negative). Moreover, for certain activation functions, we can dispense with $s$ altogether.

The function $\Sigma$ is defined by:
\[
\Sigma: x \mapsto (\sigma(x_1),\sigma(x_2),\hdots,\sigma(x_n)), 
\]
where  $\sigma:\R\to \R$ is an \emph{activation function}. We now define the class of activation functions considered in this paper.
\begin{definition}\longthmtitle{Activation functions class $ \Aquad $}\label{def:activation-class}
The activation function $\sigma:\R\to \R$ is said to be in class $ \Aquad$ if:
\begin{enumerate}
\item It satisfies a quadratic differential equation: 
\[
D\xi=a_0+a_1\xi+a_2\xi^2,
\] 
with $a_1,a_2,a_3\in \R$, $a_2\ne 0$, and $\xi=D^j\sigma$ for some $j\in\N_0$. Here, $D^j\sigma$ denotes the derivative of $\sigma$ of order $j$ and $D^0\sigma=\sigma$. 
\item The activation function $\sigma:\R\to \R$ is Lipschitz continuous, $D\sigma\ge 0$,  and $\xi=D^j\sigma$ defined above is injective.
\end{enumerate}
\end{definition}

\begin{table*}[h]
\label{Table}
\begin{center}
\begin{tabular}{c|c|c}
\textbf{Function name} & \textbf{Definition} & \textbf{Satisfied differential equation}\\\hline
Logistic function & $\sigma(x)=\frac{1}{1+e^{-x}}$ & $D\sigma-\sigma+\sigma^2=0$\\
Hyperbolic tangent & $\sigma(x)=\frac{e^x-e^{-x}}{e^x+e^{-x}}$ & $D\sigma-1+\sigma^2=0$\\
Soft plus & $\sigma(x)=\frac{1}{r}\log(1+e^{rx})$ & $D^2\sigma-rD\sigma +r (D\sigma)^2=0$\\\hline
\end{tabular}
\caption{Activation functions and the differential equations they satisfy.}
\end{center}
\end{table*}


Several activation functions used in the literature are in class $ \Aquad$ as can be seen in Table~\ref{Table}. Moreover, activation functions that are not differentiable can also be handled via approximation. For example, the ReLU function defined by $\max\{0,x\}$ can be approximated by $\sigma(x)=\log(1+e^{rx})/r$, as $r\to \infty$, which satisfies the quadratic differential equation given in Table~\ref{Table}. Similarly, the leaky ReLU, defined by $\sigma(x)=x$ for $x\ge 0$ and $\sigma(x)=r x$ for $x<0$, is the limit as $k\to \infty$ of $\alpha(x)=r x+\log(1+e^{(1-r)kx})/k$, and the function $\alpha$ satisfies $D^2\alpha-k(1+r)D\alpha+k(D\alpha)^2+k r=0$.

 The Lipschitz continuity assumption is made to simplify the presentation and can be replaced with local Lipschitz continuity, which then does not need to be assumed, since $\sigma$ is analytic in virtue of being the solution of an analytic (quadratic) differential equation. Moreover, all the activation functions in Table~\ref{Table} are Lipschitz continuous, have positive derivative and are thus injective.
 
To formally state the problem under study in this paper, we need to discuss a different point of view on the solutions of the control system~\eqref{ControlSystemModel} given by \emph{flows}. A continuously differentiable curve $x:[0,\tau]\to \R^n$ is said to be a solution of~\eqref{ControlSystemModel} under the piecewise continuous input $(s,W,b):[0,\tau]\to \R\times\R^{n\times n}\times \R^n$ if it satisfies~\eqref{ControlSystemModel}. Under the stated assumptions on $\sigma$, given a piecewise continuous input and a state $\xin\in \R^n$, there is one and at most one solution $x(t)$ of~\eqref{ControlSystemModel} satisfying $x(0)=\xin$. Moreover, solutions are defined for all $\tau\in \R_0^+$. We can thus define the flow of~\eqref{ControlSystemModel} under the input $(s,W,b)$ as the map $\phi^\tau:\R^n\to \R^n$ given by the assignment $\xin \mapsto x(\tau)$. In other words, $\phi^\tau(\xin)$ is the point reached at time $\tau$ by the unique solution starting at $\xin$ at time $0$. When the time $\tau$ is clear from context, we denote a flow simply by $\phi$. For notational purposes, we find it convenient to sometimes denote the flow defined by the solution of the differential equation $\dot{x}=Z(x)$, where $Z:\R^n\to \R^n$, by $Z^\tau$.

Before we introduced the main problem statement, we make a remark on the class of functions to be approximated.

\begin{remark}\longthmtitle{General homomorphisms} 
In Section~\ref{NNTraining}, we focused on automorphisms on $ \real^n $. This being said, in some applications we are interested in approximating an arbitrary continuous function $f:\R^n\to \R^r$ with $ r $ not necessarily equal to $ n$. Since flows have the same domain and co-domain, we can lift $f$ to a map \mbox{$\tilde{f}:  \R^k \to \R^k$}. When $n>r$, we lift $f$ to $\tilde{f}=\imath\circ f:\R^n\to \R^n$, where $\imath:\R^r\to \R^n$ is the injection given by $\imath(x)=(x_1,\hdots,x_r,0,\hdots,0)$. In this case $k=n$. When $n<r$, we lift $f$ to $\tilde{f}=f\circ \pi: \R^r\to \R^r$, where $\pi:\R^r\to \R^n$ is the projection $\pi(x_1,\hdots, x_n,x_{n+1},\hdots, x_r)=(x_1,\hdots,x_n)$. In this case $k=r$. Although we could consider factoring $f$ through a map $g:\R^n\to \R^r$, i.e., to construct $\tilde{f}:\R^n\to \R^n$ so that $f=g\circ \tilde{f}$ as done in, e.g.,~\cite{QL-TL-ZS:19}, the construction of $g$ requires a deep understanding of $f$, since a necessary condition for this factorization is $f(\R^n)\subseteq g(\R^n)$. Constructing $g$ so as to contain $f(\R^n)$ on its image requires understanding what $f(\R^n)$ is and this information is not available in learning problems. Given this discussion, in the remainder of this paper we directly assume we seek to approximate a map $f:\R^n\to \R^n$.
\end{remark}

Throughout the paper, we will investigate approximation in the sense of the $ L^\infty $ (supremum) norm, i.e., for $f:\R^n\to \R^n$, we consider:
$$\Vert f \Vert_{L^\infty(E)}=\sup_{x\in E}\vert f(x)\vert_\infty,$$
where $E\subset \R^n$ is the compact set over which is the approximation is going to be conducted and $\vert f(x)\vert_\infty=\max_{i\in\{1,\hdots,n\}}\vert f_i(x)\vert$. 

We are now ready to state the two problems we study in this paper.

\begin{problem}\longthmtitle{Memorization}\label{problem1}
Let $f: \R^n\to \R^n$ be a continuous function, $\sampleset\subset \R^n$ be a finite set, and $\varepsilon \in \R^+_0$ be the desired approximation accuracy. Does there exist a time $\tau\in \R^+$ and an input $(s,W,b):[0,\tau]\to \R\times\R^{n\times n}\times \R^n$ so that the flow $\phi^\tau:\R^n\to \R^n$ defined by the solution of~\eqref{ControlSystemModel} under the said input satisfies:
$$\Vert f- \phi^\tau\Vert_{L^\infty(\sampleset)}\le \varepsilon?$$
\end{problem}

Note that we allow $\varepsilon$ to be zero in which case the flow $\phi^\tau$ matches $f$ exactly on $\sampleset$, i.e., $f(x)=\phi^\tau(x)$ for every $x\in \sampleset$.
The next problem considers the more challenging case of approximation on compact sets and allows for residual networks with $m>n$ neurons per layer when approximating functions on $\R^n$.

\begin{problem}\longthmtitle{Universal approximation}\label{problem2}
Let $f: \R^n\to \R^n$ be a continuous function, $E\subset \R^n$ be a compact set, and $\varepsilon \in \R^+$ be the desired approximation accuracy. Does there exist  $m\in \N$, a time $\tau\in \R^+$, an injection $\alpha:\R^n\to \R^{m}$, a projection $\beta:\R^{m}\to \R^n$, and an input $(s,W,b):[0,\tau]\to \R\times\R^{m\times m}\times \R^m$ such that the flow $\phi^\tau:\R^m\to \R^m$ defined by the solution of~\eqref{ControlSystemModel}  under the said input satisfies:
$$\Vert f- \beta\circ \phi^\tau\circ\alpha\Vert_{L^\infty(E)}\le \varepsilon?$$
\end{problem}
 \begin{figure}[htb!]
   \centering 
\includegraphics{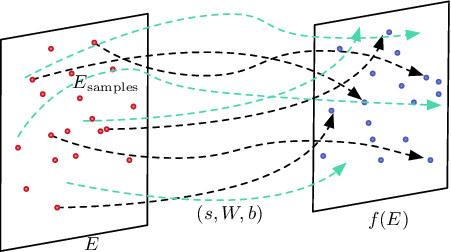}
     \caption{A schematic picture of the two problems under study. The objective in Problem~\ref{problem1} is to find a control input 
     $ (s,W,b) $ that takes all the sample points in $ \sampleset $ (shown by red dots) to their corresponding points in $ f(E) $ (shown by blue dots); some of these assignments are shown by black dashed arrows. In Problem~\ref{problem2} we are additionally concerned with the quality of the assignments of the points outside $ \sampleset $, shown by green dashed arrow, as we aim to guarantee approximation in the uniform sense.      
     }\label{fig1}
\end{figure}
Figure~\ref{fig1} provides a schematic picture of Problems~\ref{problem1} and~\ref{problem2}. In the next section, we will show the answer to both these problems to be affirmative. 

\newpage
\section{Controllability capabilities\\of deep residual networks} \label{sec:controllability}

We first discuss the problem of constructing an input for~\eqref{ControlSystemModel} so that the resulting flow $\phi$ satisfies $\phi(x)=f(x)$ for all the points $x$ in a given finite set $\sampleset\subset\R^n$. As we explained in Section~\ref{NNTraining}, this is equivalent to determining if the ensemble control system~\eqref{Product} is controllable. 

\begin{definition}\longthmtitle{Controllability}
A point $\Xf\in \R^{n\times d}$ is said to be reachable from a point $\Xin\in \R^{n\times d}$, for the control system~\eqref{Product}, if there exist $\tau\in \R^+$ and a control input $(s,W,b):[0,\tau]\to \R\times \R^{n\times n}\times \R^n$ so that the solution $X$ of~\eqref{Product} under said input satisfies $X(0)=\Xin$ and $X(\tau)=\Xf$. Control system~\eqref{Product} is said to be controllable on a submanifold $M$ of $\R^{n\times d}$ if any point in $M$ is reachable from any point in $M$.
\end{definition}

It is simple to see that controllability of~\eqref{Product} cannot hold on all of $ \R^{n\times d}  $ since, if the initial state $X(0)$ satisfies $X_{\bullet i}(0)=X_{\bullet j}(0)$ for some $i\ne j$, we must have $X_{\bullet i}(t)=X_{\bullet j}(t)$ for all $t\in [0,\tau]$ by uniqueness of solutions of differential equations.

Our first result establishes that controllability holds for the ensemble control system~\eqref{Product} on a dense and connected submanifold of $\R^{n\times d}$, independently of the (finite) number of copies $d$, as long as the activation function is in $\Aquad$.

\begin{theorem}\longthmtitle{Controllability on a submanifold under $\Aquad$}\label{Thm:Controllability}
Consider the set $N\subset \R^{n\times d}$ defined by:
\[
N=\{A\in  \R^{n\times d} \ \vert \prod_{1\le i<j\le d}(A_{\ell i}-A_{\ell j})=0, \ \ell \in\{1,\ldots, n\}\}.
\]
Let $ n>1 $ and suppose the activation function is in $\Aquad$.
Then, the ensemble control system~\eqref{Product} is controllable on the submanifold $M=\R^{n\times d}\backslash N$.
\end{theorem}

It is worth mentioning that the assumption of $ n\neq 1 $ ensures connectedness of the submanifold $ M $, which we rely on to obtain controllability. 

The proof of Theorem~\ref{Thm:Controllability} uses several key ideas from geometric control that we now review. A collection of vector fields $\mathcal{F}=\{Z_1,\hdots,Z_k\}$ on a manifold $M$ is said to be controllable if given $\xin,\xf\in M$, there exists a finite sequence of times $0<t_1<t_1+t_2<\hdots<t_1+\hdots+t_q$ so that:
$$Z_\ell^{t_q}\circ\hdots\circ Z_2^{t_2}\circ  Z_1^{t_1}(\xin)=\xf,$$
where $Z_i\in\mathcal{F}$ and $Z_i^t$ is the flow of $Z_i$. When the vector fields $Z_i$ are smooth, $M$ is smooth and connected, and the collection $\mathcal{F}$ satisfies:
$$Z\in \mathcal{F}\implies \alpha Z\in \mathcal{F}\text{ for some }\alpha<0,$$
then $\mathcal{F}$ is controllable provided the evaluation of the Lie algebra generated by $\mathcal{F}$ at every point $x\in M$ has the same dimension as $M$, see, e.g.,~\cite{jurdjevic_1996}\footnote{In this footnote we provide additional details relating controllability of a family of vector fields to the Lie algebra rank condition. Let us denote by $\mathcal{A}_\mathcal{F}(x)$ the reachable set of the family of smooth vector fields $\mathcal{F}$ from $x\in M$, i.e., the set of all points $\xf\in \R^n$ of the form:
$$\xf=Z_\ell^{t_q}\circ\hdots\circ Z_2^{t_2}\circ  Z_1^{t_1}(x),$$
for $Z_i\in\mathcal{F}$ and $0<t_1<t_1+t_2<\hdots<t_1+\hdots+t_q,$,
and denote by $Lie_x(\mathcal{F})$ the evaluation of the Lie algebra generated by $\mathcal{F}$ at $x\in M$. By $\mathcal{F}'$ we denote the family of vector fields of the form $\sum_i \lambda_i X_i$ with $X_i\in \mathcal{F}$ and $\lambda_i\ge 0$. Since $\mathcal{F}\subseteq \mathcal{F'}$ we have $\mathcal{A}_\mathcal{F}(x)\subseteq \mathcal{A}_\mathcal{F'}(x)$. By Theorem 8 in Chapter 3 of~\cite{jurdjevic_1996} we have that:
$$\mathcal{A}_\mathcal{F}(x)\subseteq \mathcal{A}_\mathcal{F'}(x)\subseteq \cl (\mathcal{A}_\mathcal{F}(x)),$$
where $\cl$ denotes topological closure. Moreover, by Theorem 2 in Chapter 3 of~\cite{jurdjevic_1996}, if $Lie_x(\mathcal{F})=T_xM$ for every $x\in M$, then $\mathrm{int}(\cl(\mathcal{A}_\mathcal{F}(x)))=\mathrm{int}(\mathcal{A}_\mathcal{F}(x))$. We thus obtain:
$$\mathrm{int}(\mathcal{A}_\mathcal{F}(x))\subseteq \mathrm{int}(\mathcal{A}_\mathcal{F'}(x))\subseteq \mathrm{int}(\cl (\mathcal{A}_\mathcal{F}(x)))=\mathrm{int} (\mathcal{A}_\mathcal{F}(x)).$$
But if $\mathcal{F'}$ is controllable, $\mathrm{int} (\mathcal{A}_\mathcal{F}'(x))=M$ and thus $\mathcal{F}$ is also controllable. Therefore, we now focus on determining if $\mathcal{F}'$ is controllable. Provided that for each $X\in \mathcal{F}$ there exists $X'\in \mathcal{F}$ satisfying $X=\sigma X'$ with $\sigma<0$ (this is weaker than symmetry, symmetry is this property for $\sigma=-1$), $\mathcal{F}'$ is simply the vector space spanned by $\mathcal{F}$. Moreover, since the control system $\dot{x}=\sum_i X_i u_i$ with $X_i\in \mathcal{F}$ and $u_i\in \R$ generates the same family of vector fields as $\mathcal{F}'$, we conclude that we can instead study the reachable set of $\dot{x}=\sum_i X_i u_i$ with $X_i\in \mathcal{F}$ which is driftless. By Theorem 2 in Chapter 4,  in~\cite{jurdjevic_1996} the control system $\dot{x}=\sum_i X_i u_i$ is controllable provided that $Lie_x(\mathcal{F'})=T_x M$  for every $x\in M$.}. Recall that the Lie algebra generated by $\mathcal{F}$, and denoted by $Lie(\mathcal{F})$, is the smallest vector space of vector fields on $M$ containing $\mathcal{F}$ and closed under the Lie bracket. By evaluation of $Lie(\mathcal{F})$ at $x\in M$, we mean the finite-dimensional vector subspace of the tangent space of $M$ at $x$ that is obtained by evaluating every vector field in $Lie(\mathcal{F})$ at $x$. The proof consists in establishing controllability by determining the points at which $Lie(\mathcal{F})$ has the right dimension for a collection of vector fields $\mathcal{F}$ induced by the ensemble control system~\eqref{Product}. Theorem~\ref{Thm:Controllability} can now be proved by making use of these concepts.

\begin{proof}Consider the control system given in~\eqref{Product}. We prove that under the mentioned assumptions, there is a choice of the control inputs $(s,W,b)$ that renders~\eqref{Product} controllable in $ M $.  
It will be sufficient to work with inputs that are piecewise constant, and we can further simplify the analysis by choosing the family of inputs $(s,W,b)$ given by~\eqref{Inputs1} and~\eqref{Inputs2}, where:
\begin{itemize}
\item the first class of inputs is given by:
\begin{equation}
\label{Inputs1}
(\pm 1,0,c e_j),
\end{equation}
where $ j\in \{1,2,\hdots,n\} $ and $c\in \R$ is any value such that $\sigma(c)\ne 0$ and $e_j\in \R^n$ has zeros in all its entries except for a $1$ on its $j$th entry;
\item the second class of inputs is given by:
\begin{equation}
\label{Inputs2}
(\pm 1, E_{jk},0),
\end{equation}
where  $ j,k \in\{1,2,\hdots,n\} $ and $E_{ij}$ is the $n\times n$ matrix that has zeros in all its entries except for a $1$ in its $j$th row and $k$th column.
\end{itemize}
Once we substitute these inputs into the right hand side of the ensemble control system~\eqref{Product}, we obtain a family of vector fields on $\R^{n\times d}$. More specifically, the vector fields arising from the inputs~\eqref{Inputs1}, denoted by 
$\{X^{\pm}_j\}_{j\in \{1,\hdots,n\}}$, are given by:
\begin{equation}
\label{VFX}
X_j^+=\sigma(c)\sum_{i=1}^d \pder{}{A_{ji}} \quad \mathrm{and} \quad X_j^-=- X_j^+.
\end{equation}
Similarly, the vector fields arising from  the inputs~\eqref{Inputs2}, denoted by $\{Y^{\pm}_{j,k}\}_{j,k\in\{ 1,\hdots,n\}}$, are given by:
\begin{equation}
\label{VFY}
Y_{jk}^+=\sum_{i=1}^d\sigma\left(A_{ki}\right)\frac{\partial }{\partial A_{ji}}  \quad \mathrm{and} \quad Y_{jk}^-=-Y_{jk}^+.
\end{equation}
This definition abuses notation, since defining a vector field on $\R^{n\times d}$ requires one summation over $i$ and one over $j$. However, summation over $j$, i.e., summation over rows, only produces non-zero terms for one row, that we decided to index by $j$.

We make the observation that, since $\sigma(c)\ne 0$, we can simplify the vector fields $X_j^\pm$ to:
\[
X_j^+=\sum_{i=1}^d \pder{}{A_{ji}} \quad \mathrm{and} \quad X_j^-=-X_j^+,
\]
without altering controllability. This follows from the observation that for any vector field $X$ with flow $X^t$ we have $X^{\alpha \tau}=(\alpha X)^\tau$ for any $\alpha\in \R$.

By Proposition~\ref{Prop:Connected}, $ M $ is a connected smooth submanifold of $ \R^{n\times d} $. The remainder of the proof consists of showing that the family of vector fields $\mathcal{F}=\{X^{\pm}_j,Y^{\pm}_{jk}\}_{j,k\in \{1,\hdots,n\}}$, restricted to $M$, is controllable on $M$. As discussed prior to this proof, since these vector fields in $\mathcal{F}$ are smooth and satisfy $Z\in \mathcal{F}\implies -Z\in\mathcal{F}$, it suffices to establish that $\dim (Lie_A(\mathcal{F}))=\dim (M)=nd$ for every $A\in M$ and where $ Lie_A(\mathcal{F})$ denotes the evaluation at $A$ of the Lie algebra generated by $\mathcal{F}$.

We generate $Lie(\mathcal{F})$ by iteratively computing Lie brackets. For two vector fields $X $ and $Y $ on $ \R^{n\times d} $, we use the notation $\text{ad}_X Y=[X,Y]$ and $\text{ad}_X^{\ell+1} Y=[X,\text{ad}_X^{\ell} Y]$ where $[X,Y]$ denotes the Lie bracket between $X$ and $Y$. For our purpose, it is enough to compute $ \text{ad}^\ell_{X_k^\pm} {Y_{jk}^\pm} $ and, given the implication $Z\in\mathcal{F}\implies -Z\in\mathcal{F}$, it suffices to compute:
\begin{equation}
\label{LieBracket}
(\text{ad}^\ell_{X_k^+} {Y_{jk}^+})(A)=\sum_{i=1}^d D^{\ell}\sigma(A_{ki})\frac{\partial }{\partial A_{ji}}.
\end{equation}

In order to show that $\dim(Lie_A(\mathcal{F}))=\dim(M)$ at every $A\in M$, we find it convenient to work with the vectorization of elements of $ \R^{n\times d} $. In particular, we associate the vector $ \vect(A) \in \R^{nd} $ to each matrix $ A \in \R^{n\times d} $ where the entry $ (i,j) $ of $A$ is identified with the entry $ d^{i-1}+j $ of $ \vect(A) $. For a collection of matrices $\{A_1,\ldots, A_k\} $, we denote by $ \vect\{A_1,\ldots, A_k\}$ the collection of vectors $\vect\{A_1,\ldots, A_k\}=\{ \vect(A_1),\ldots,  \vect(A_k)\} $.

Consider now the indexed collection of vector fields $\mathcal{S}=\{Z_\ell\}_{\ell\in\{1,\hdots,n^2(d-1)\}}$ where:
\begin{align*}
&Z_{1+(j-1)(n^2+1) }=\vect(X_j),\\
&Z_{1+i+kn+(j-1)(n^2+1)}=\vect(\text{ad}_{X_{j}}^k Y_{ji}).
\end{align*}
We note that every $Z\in \mathcal{S}$ belongs to $Lie(\mathcal{F})$ since the vector fields in $\mathcal{S}$ either belong to $\mathcal{F}$ or are obtained by computing Lie brackets between elements of $\mathcal{F}$ and elements of $\mathcal{S}$. Moreover, we claim the evaluation of the vector fields in $\mathcal{S}$ at every $A\in M$ results in $nd$ linearly independent vectors. To establish this claim, we form the matrix:
\begin{align*}
G(&\vect(A))=\\
&\left[ Z_1(\vect(A))\vert Z_2(\vect(A))\vert \hdots\vert Z_{n^2(d-1)}(\vect(A)) \right],
\end{align*}
and note that a simple but tedious computation, using~\eqref{LieBracket}, shows that $G$ is a block diagonal matrix with $d$ blocks, as shown in~\eqref{eq:aux}. Here, we assume that $ d\geq n $, i.e., the number of samples is larger than the dimension of the domain of the function to be approximated. Note that if $ d<n $, then the set of available controls are overparameterized and hence one can select inputs to generate any direction in $T_xM$ without relying on higher order Lie brackets.

\begin{figure*}
\begin{align}\label{eq:aux}
&\subscr{G}{blk}(\vect(A))=\cr
&\begin{bmatrix}
1 & \sigma(A_{11}) & \cdots & \sigma(A_{1n}) &D\sigma( A_{11} ) & \cdots & D\sigma( A_{1n} ) & D^{d-2}\sigma( A_{11} ) &\cdots& D^{d-2}(A_{1n})\cr
1 & \sigma(A_{21}) & \cdots & \sigma(A_{2n}) &D\sigma( A_{21} ) & \cdots & D\sigma( A_{2n} ) & D^{d-2}\sigma( A_{21} ) &\cdots& D^{d-2}(A_{2n})\cr
\vdots &\vdots && \vdots  & \vdots&&\vdots&\vdots&&\vdots \cr
1 & \sigma(A_{n1}) & \cdots & \sigma(A_{nn}) &D\sigma( A_{n1} ) & \cdots & D\sigma( A_{nn} ) & D^{d-2}\sigma( A_{n1} ) &\cdots& D^{d-2}(A_{nn}) 
\end{bmatrix}.    
\end{align}
\end{figure*}
To finish the proof, it suffices to show that $\subscr{G}{blk}$ has rank $n$ (since it has $n$ rows) and this is accomplished by showing there is a choice of $n$ columns that are linearly independent. Since $A\in M$ implies $A\notin N$, by definition, there exists $\ell\in \{1,\hdots,n\}$ such that:
$$\prod_{1\le i<j\le d}(A_{\ell i}-A_{\ell j})\ne 0.$$
Moreover, by our assumption on injectivity of $\sigma$, we conclude that:
$$\prod_{1\le i<j\le d}(\sigma(A_{\ell i})-\sigma(A_{\ell j}))\ne 0,$$
and it follows from Lemma~\ref{LemmaDet} that the matrix:
\begin{equation}
\begin{bmatrix}
\label{ForBahman}
1 & \sigma(A_{1\ell}) &D\sigma( A_{1\ell} ) & \cdots  & D^{d-2}\sigma( A_{1\ell} )\\
1 & \sigma(A_{2\ell}) &D\sigma( A_{2\ell} ) & \cdots  & D^{d-2}\sigma( A_{2\ell} )\\
\vdots &\vdots && \vdots \\
1 & \sigma(A_{n\ell}) &D\sigma( A_{n\ell} ) & \cdots  & D^{d-2}\sigma( A_{n\ell} )\\
\end{bmatrix},
\end{equation}
has rank $n$, i.e., for every $A\in M$ there exists $n$ columns of $\subscr{G}{blk}(\vect(A))$ that are linearly independent. The proof is then complete by noting that for $n>1$, $M$ is connected, as asserted by Proposition~\ref{Prop:Connected}.
\end{proof}

The preceding proof used the controllability properties of the vector fields~\eqref{VFX} and~\eqref{VFY};  upon a closer look, the reader can observe that it suffices for $s$ to take values in the set $\{-1,1\}$ (or any set with two elements, one being positive and one being negative), for $W$ to take values on $\{1,0\}$ (or any other set $\{0,c\}$ with $c\ne 0$) and for $b$ to take values on $\{0,d\}$ for some $d\in \R$ such that $\sigma(d)\ne 0$. Moreover, when the activation function is and odd function, i.e., $\sigma(-x)=-\sigma(x)$, as is the case for the hyperbolic tangent, the conclusions of Theorem~\ref{Thm:Controllability}  hold for the simpler version of~\eqref{ControlSystemModel}, where we fix $s$ to be $1$. 
Taking these observations one step further, one can establish controllability of an alternative network architecture defined by:
$$\dot{x}=S\Sigma(x)+b,$$
where the $n\times n$ matrix $S$ and the $n$ vector $b$ only need to assume values in a set of the form $\{c^-,0,c^+\}$ where $c^-\in \R^-$ and $c^+\in \R^+$.

The following corollary of Theorem~\ref{Thm:Controllability} weakens controllability to reachability but applies to a larger set.
 
 \begin{corollary}\longthmtitle{Reachability on a submanifold under $\Aquad$}
\label{Cor:Reachability}
Let $M\subset \R^{n\times d}$ be the submanifold defined in Theorem~\ref{Thm:Controllability}. Under assumptions of Theorem~\ref{Thm:Controllability}, any point in $M$ is reachable from a point $A\in \R^{n\times d}$ for which:
\[
A_{\bullet i}\ne A_{\bullet j},
\]
holds for all $ i\ne j $, where $ i,j\in\{1,\hdots,d \}$.
\end{corollary}
\begin{proof}The result follows from Theorem~\ref{Thm:Controllability} once we establish the existence of a solution of~\eqref{Product} taking $\Xin$ to some point $\Xf\in M$. This is because Theorem~\ref{Thm:Controllability} states that any other point in $M$ will then be reachable. We proceed by showing the existence of a solution taking $\Xin$ to a point $\Xf$ satisfying $\Xf_{1i}\ne \Xf_{1j}$ for all $i\ne j$, $i,j\in\{1,\hdots,d\}$. Clearly, $\Xf\in M$.

Assume, without loss of generality, that $\Xin_{11}=\Xin_{12}$. We will design an input, for a duration $\tau>0$, that will result in a solution $X(t)$ with $X_{11}(\tau)\ne X_{12}(\tau)$, while ensuring that if $\Xin_{1i}$ is different from $\Xin_{1j}$ then $X_{1i}(\tau)$ is different from $X_{1j}(\tau)$.

By assumption, $\Xin_{\bullet 1}\ne \Xin_{\bullet 2}$. Hence, there must exist $k\in\{1,\hdots,n\}$ so that $\Xin_{k 1}\ne \Xin_{k 2}$. We use $k$ to define the input $s=1$, $b=0$, and the matrix $W$ all of whose entries are zero except for $W_{1k}$ that is equal to $1$. This choice of input results in the solution:
$$X(t)=\Xin+t\begin{bmatrix}
\sigma(\Xin_{k1}) & \sigma(\Xin_{k2}) & \hdots & \sigma(\Xin_{kd})\\
0 & 0 & \hdots & 0\\
\vdots & \vdots & & \vdots\\
0 & 0 & \hdots & 0
\end{bmatrix}.$$
We note that 
\[
\frac{d}{dt}\big\vert_{t=0}(X_{11}(t)-X_{12}(t))=\sigma(\Xin_{k1})-\sigma(\Xin_{k2})\ne 0,
\] 
since $\sigma$ is injective. Therefore, there exists $\tau_1\in \R^+$ such that $X_{11}(t)-X_{12}(t)\ne 0$ for all $t\in ]0,\tau_1]$, i.e., $X_{11}(t)\ne X_{12}(t)$ for all $t\in ]0,\tau_1]$. Moreover, we now show existence of $\tau_2$ so that for all $t\in [0,\tau_2]$ we have   $X_{1i}(t)\ne X_{2j}(t)$ whenever $X_{1i}(0)=\Xin_{1i}\ne \Xin_{2j}=X_{2j}(0)$. For a particular pair $(X_{1i},X_{2j})$ for which $\Xin_{1i}\ne \Xin_{2j}$, the equality $\Xin_{1i}+t\sigma(\Xin_{ki}(0))=\Xin_{1j}+t\sigma(\Xin_{kj}(0))$ defines the intersection of two lines. If they intersect for positive $t$, say $t_2$, it suffices to choose $\tau_2$ smaller $t_2$. Moreover, by choosing $\tau_2$ to be smaller than the positive intersection points for all pairs of lines corresponding to all pairs $(X_{1i},X_{2j})$ for which $\Xin_{1i}\ne \Xin_{2j}$, we conclude that for all $t\in [0,\tau_2]$, $X_{1i}(0)=\Xin_{1i}\ne \Xin_{2j}=X_{2j}(0)$ implies $X_{1i}(t)\ne X_{2j}(t)$. Let now $\tau=\min\{\tau_1,\tau_2\}$. The point $X(\tau)$ satisfies the two properties we set to achive: 1) $X_{11}(\tau)\ne X_{12}(\tau)$; and 2) $X_{1i}(\tau)\ne X_{1j}(\tau)$ if $\Xin_{1i}\ne \Xin_{1j}$. 

By noticing that $\Xin_{ij}=X_{ij}(\tau)$ for $i>1$ and any $j\in\{1,\hdots, d\}$, we can repeat this process iteratively to force all the entires of the first row of $X$ to become different, the same way we forced the first two. 
\end{proof}

The assumption $A_{\bullet i}\ne A_{\bullet j}$ in Corollary~\ref{Cor:Reachability} requires all the columns of $A$ to be different and is always satisfied when $A=\left[x^1\vert x^2\vert\hdots\vert x^d\right]$, $x^i\in \sampleset$. Hence, for any finite set $\sampleset$ there exists a flow $\phi$ of~\eqref{ControlSystemModel} satisfying $f(x)=\phi(x)$ for all $x\in \sampleset$ provided that $f(\sampleset)\subset M$, i.e., Problem~\ref{problem1} is solved with $\varepsilon=0$. Moreover, since $M$ is dense in $\R^{n\times d}$, when $f(\sampleset)\subset M$ fails, there still exists a flow $\phi$ of~\eqref{ControlSystemModel} taking $\phi(x)$ arbitrarily close to $f(x)$ for all $x\in \sampleset$, i.e., Problem~\ref{problem1} is solved for any $\varepsilon>0$. This result also sheds light on the memorization capacity of residual networks as it states that almost any finite set of samples can be memorized, independently of its cardinality. See, e.g.,~\cite{DBLP:conf/nips/YunSJ19,vershynin2020memory}, for recent results on this problem that do not rely on differential equation models.

Some further remarks are in order. Theorem~\ref{Thm:Controllability} and Corollary~\ref{Cor:Reachability} do not directly apply to the ReLU activation function, defined by $\max\{0,x\}$, since this function is not differentiable. However, the ReLU is approximated by the activation function:
$$\frac{1}{r}\log(1+e^{r x}),$$
as $r\to\infty$. In particular, as $r\to\infty$ the ensemble control system~\eqref{Product} with $\sigma(x)=\log(1+e^{r x})/r$ converges to the ensemble control system~\eqref{Product} with $\sigma(x)=\max\{0,x\}$ and thus the solutions of the latter are arbitrarily close to the solutions of the former whenever $r$ is large enough. Moreover, $\xi=D\sigma$ satisfies $D\xi=r\xi-r\xi^2$ and $D\xi=re^{rx}/(1+e^{rx})^2>0$ for $x\in \R$ and $r>0$ thus showing that $\xi$ is an increasing function and, consequently, injective.


\section{Approximation capabilities of ResNets}\label{Sec:FunctionApproximation}

In order to extend the approximation guarantees from a finite set $\sampleset\subset\R^n$ to an arbitrary compact set $E\subset \R^n$, we rely on the notion of monotonicity. On $\R^n$ we consider the ordering relation $x\preceq x'$ defined by $x_i\le x_i'$ for all $i\in\{1,\hdots,n\}$ and $x,x'\in \R^n$. A map $f:\R^n\to \R^n$ is said to be monotone when it respects this ordering relation, i.e., when $x\preceq x'$ implies $f(x)\preceq f(x')$. When $f$ is continuously differentiable, monotonicity admits a simple characterization~\cite{hirsch2006monotone}:
\begin{equation}
\label{ConditionMonotoneMap}
\frac{\partial f_i}{\partial x_j}\ge 0,\quad \forall i,j\in\{1,\hdots,n\}.
\end{equation}
A vector field $Z:\R^n\to \R^n$ is said to be monotone when its flow $\phi^\tau:\R^n\to \R^n$ is a monotone map. Monotone vector fields admit a characterization similar to~\eqref{ConditionMonotoneMap}, see~\cite{Book:Monotone}:
\begin{equation}
\label{ConditionMonotoneVF}
\frac{\partial Z_i}{\partial x_j}\ge 0,\quad \forall i,j\in\{1,\hdots,n\}, i\ne j.
\end{equation}

\subsection{Universal approximation results}

Our first result shows that when the activation functions belong to the class $\Aquad$, we can uniformly approximate functions that can be connected to the identity function through an analytic and monotone homotopy, i.e., for which there exists an analytic function $h:\R^n\times [0,1]\to\R^n$ satisfying $h(x,0)=x$, $h(x,1)=f(x)$, and $h_\tau(x)\preceq h_\tau(x')$ for some $\tau\in [0,1]$ implies $h_{t}(x)\preceq h_{t}(x')$ for all $t\in [\tau,1]$. Note this implies, for $t=1$, that $h_1=f$ is monotone and analytic.

\begin{theorem}\longthmtitle{Universal approximation of monotone analytic functions homotopic to the identity under $\Aquad$}
\label{Thm:MonotoneFlow}
Let $ n>1 $ and suppose the activation function is in $\Aquad$.
Then, for every function $f:\R^n\to \R^n$ that can be connected to the identity function through an analytic and monotone homotopy, for every compact set $E\subset \R^n$, and for every $\varepsilon \in \R^+$ there exist a time $\tau\in \R^+$ and an input $(s,W,b):[0,\tau]\to \R\times\R^{n\times n}\times \R^n$ so that the flow $\phi^\tau:\R^n\to \R^n$ defined by the solution of~\eqref{ControlSystemModel} with state space $\R^n$ under the said input satisfies:
\begin{equation}
\label{Guarantee}
\Vert f- \phi^\tau\Vert_{L^\infty(E)}\le \varepsilon .
\end{equation}
\end{theorem}

Our second result, which can be thought of as the main contribution of this paper, extends Theorem~\ref{Thm:MonotoneFlow} to any continuous function. To foreshadow a key part of its statement, note that we can approximate a continuous function by a polynomial using the Stone-Weierstass Theorem, and then embed the outcome into a function, defined on $\R^{2n+1}$, that can be connected to the identity by an analytic and monotone homotopy. \sout{Although such embeddings typically require doubling the dimension, see for instance}~\cite{MonotoneEmbedding}\sout{, we leverage the specific framework of the problem at hand to introduce a novel embedding that only requires increasing $n$ to $n+1$.}

\begin{corollary}\label{theorem:n+1}
\longthmtitle{Universal approximation of any continuous functions under $\Aquad$}
Let $ n>1 $ and suppose the activation function is in $\Aquad$.
Then, for every continuous function $f: \R^n\to \R^n$, for every compact set $E\subset\R^n$, and for every $\varepsilon \in \R^+$ there exist a time $\tau\in \R^+$, an injection $\alpha:\R^n\to \R^{2n+1}$, a projection $\beta:\R^{2n+1}\to \R^n$,  and an input $(s,W,b):[0,\tau]\to \R\times\R^{(2n+1)\times (2n+1)}\times \R^{2n+1}$ so that the flow $\phi^\tau:\R^{2n+1}\to \R^{2n+1}$ defined by the solution of~\eqref{ControlSystemModel} with state space $\R^{2n+1}$ under the said input satisfies:
$$\Vert f- \beta\circ \phi^\tau\circ \alpha\Vert_{L^\infty(E)}\le \varepsilon .$$
\end{corollary}

It is worth pointing out again that Corollary~\ref{theorem:n+1} improves upon the results of~\cite{PK-TL:20} \sout{and~}\cite{park2021minimum}\sout{. First of all, in contrast to the aforementioned work, this result applies to residual neural networks. In addition, in comparison to}~\cite{PK-TL:20} where a width of $ 2n+2 $ neurons is used \sout{, this result requires $ n+1 $ neurons} by lowering the number of neurons per layer to $2n+1$. With respect to~\cite{park2021minimum}, our result applies to a wider set of activation functions and, notably, to residual neural networks although requiring a larger number of neurons per layer in virtue of approximating functions by flows.

\subsection{Technical proofs}

We now start the process of proving Theorem~\ref{Thm:MonotoneFlow} and Corollary~\ref{theorem:n+1} by stating and proving a technical lemma that identifies monotonicity as a key property to establish function approximability in an $L^\infty$ sense.

\begin{lemma}
\longthmtitle{A generalization type result}
\label{Lemma:MonotoneApprox}
Let $f:\R^n\to \R^{r}$ be a continuous map and $E\subset \R^n$ a compact set. Suppose 
$\sampleset\subset \R^n$ is a finite set satisfying:
\begin{align}
&\forall x\in E\quad \exists\, \underline{x},\overline{x}\in \sampleset,\cr
&\qquad  \vert \underline{x}-\overline{x}\vert_\infty\le \delta \quad\land\quad \underline{x}_i\le x_i\le \overline{x}_i\label{PropertySampleSet},
\end{align}
for all $  i\in\{1,\hdots,n\} $, with $\delta\in \R^+$, and that $\phi:\R^n\to\R^{r}$ is a monotone map satisfying: 
\begin{equation}
\label{PropertyMonotoneMap}
\Vert f-\phi\Vert_{L^\infty(\sampleset)}\le \zeta,
\end{equation}
with $\zeta\in \R^+$. Then, we have that:
$$\Vert f-\phi\Vert_{L^\infty(E)} \le 2\omega_f(\delta)+3\zeta,$$
where $\omega_f$ is a modulus\footnote{Note that $f$, being continuous, is uniformly continuous on any  compact set.} of continuity of $f$.
\end{lemma}

\begin{proof}
The result is established by direct computation:
\begin{align*}
\vert f(x)-\phi(x)\vert_\infty & \le \vert f(x)-\phi(\underline{x})\vert_\infty+\vert \phi(\underline{x})-\phi(x)\vert_\infty\notag\\
\le& \vert f(x)-f(\underline{x})\vert_\infty+\vert f(\underline{x})-\phi(\underline{x})\vert_\infty\\
&+\vert \phi(\underline{x})-\phi(x)\vert_\infty\notag\\
\le& \omega_f(\vert x-\underline{x}\vert_\infty)+\zeta+\vert \phi(\underline{x})-\phi(x)\vert_\infty\notag\\
\le& \omega_f(\vert x-\underline{x}\vert_\infty)+\zeta+\vert \phi(\underline{x})-\phi(\overline{x})\vert_\infty\notag\\
\le& \omega_f(\vert x-\underline{x}\vert_\infty)+\zeta+\vert f(\underline{x})-f(\overline{x})\vert_\infty\\
&+\vert \phi(\underline{x})-f(\underline{x})\vert_\infty+\vert f(\overline{x})-\phi(\overline{x})\vert_\infty\notag\\
\le& \omega_f(\vert x-\underline{x}\vert_\infty)+\vert f(\underline{x})-f(\overline{x})\vert_\infty+3\zeta\notag\\
\le & \omega_f(\vert \overline{x}-\underline{x}\vert_\infty)+ \omega_f(\vert \overline{x}-\underline{x}\vert_\infty)+3\zeta \\
\le& 2\omega_f(\delta)+3\zeta,\notag
\label{MonotoneErrorBound}
\end{align*}
where we used~\eqref{PropertyMonotoneMap} to obtain the third and sixth inequalities. The fourth inequality was obtained by using monotonicity of $\phi$ to conclude $ \phi(\underline{x})\preceq \phi(x)\preceq  \phi(\overline{x})$ from $\underline{x}\preceq x\preceq \overline{x}$. 
\end{proof}

The next result shows that by restricting the input function $W$ to assume values on the set of diagonal matrices leads to controllability being restricted to a smaller set but with the benefit of the resulting flows being monotone.
\begin{proposition}
\longthmtitle{Generating monotone flows}
\label{Prop:Monotone}
Suppose the activation function is in $\Aquad$.
Then, the ensemble control system~\eqref{Product}, with the image of $W$ restricted to the class of diagonal matrices, is controllable on any connected component of the manifold:
$$M=\{A\in \R^{n\times d}\,\,\vert \prod_{1\le i<j\le d}(A_{\ell i}-A_{\ell j})\ne 0,\   \ell\in\{1,\hdots,n\}\}.$$
Moreover, the flow of~\eqref{Product} joining two states in the same connected component of $M$ is monotone.
\end{proposition}
\begin{proof}
Since the proof of this result is analogous to the proof of Theorem~\ref{Thm:Controllability} we discuss only where it differs. The restriction to the set of diagonal matrices restricts the famility of vector fields $\mathcal{F}$ in the proof of Theorem~\ref{Thm:Controllability} to $\{X_j^\pm, Y_{jj}^\pm\}_{j\in\{1,\hdots,n\}}$. Computing the matrix $G(\vect(A))$, we still obtain a block diagonal matrix but its blocks are now distinct and given by:
\begin{equation}
G_\ell (\vect(A))=\begin{bmatrix}
1 & \sigma(A_{1\ell})  &D\sigma( A_{1\ell} ) & \hdots & D^{d-2}\sigma( A_{1\ell} )\\
1 & \sigma(A_{2\ell})  &D\sigma( A_{2\ell} ) & \hdots & D^{d-2}\sigma( A_{2\ell} )\\
\vdots &\vdots & \vdots  && \vdots \\
1 & \sigma(A_{n\ell})  &D\sigma( A_{n\ell} ) & \hdots & D^{d-2}\sigma( A_{n\ell} )
\end{bmatrix},\notag 
\end{equation}
where $ \ell\in\{1,\hdots,n\} $.
It now follows from injectivity of $\sigma$, Lemma~\ref{LemmaDet}, and the definition of $M$ that all these matrices are of full rank and we conclude controllability. Moreover, since the employed vector fields satisfy~\eqref{ConditionMonotoneVF}, they are monotone. Hence, the resulting flow is also monotone. 
\end{proof}

\begin{proof}[Theorem~\ref{Thm:MonotoneFlow}] 
The result to be proved will follow at once from Lemma~\ref{Lemma:MonotoneApprox} when we show existence of a finite set $\sampleset$ and a flow $\phi$ of~\eqref{ControlSystemModel} satisfying its assumptions. Existence of $\phi$ will be established by constructing an input $(s,W,b):[0,\tau]\to \R\times \R^{n\times n}\times \R^n$ that is piecewise constant. While the input is held constant, the righthand side of~\eqref{ControlSystemModel} is a vector field, which we prove to be monotone. Since the composition of monotone flows is a monotone flow, the desired monotonicity of $\phi$ ensues. 

Let $\sampleset=\{x^1,x^2,\hdots,x^d\}\subset \R^n$ satisfy~\eqref{PropertySampleSet} for a constant $\delta\in \R^+$ to be later specified. \sout{Define the map $h:\R^{n}\times[0,1]\to \R$ by
$$h(x,t)=(1-t)x+tf(x)$$
and note that for $t\in [0,1]$ the map $h_t:\R^n\to \R^n$ defined by $h_t(x)=h(x,t)$ is monotone. This can be verified using criterion~\eqref{ConditionMonotoneMap} and noticing the inequalities $t\ge0$ and $(1-t)\ge 0$ hold for all $t\in [0,1]$.} Define the curve $Y:[0,1]\to \R^{n\times d}$ obtained by applying $h$ to every point in $\sampleset$:
\begin{equation}
\label{FunctionY}
Y(t)=\left[h(x^1,t)\vert h(x^2,t)\vert\hdots\vert h(x^d,t)\right],
\end{equation}
where $h$ is the monotone homotopy from the identity to $f$.
Since $h_1(x)=h(x,1)=f(x)$, we will show that for every $\zeta \in \R^+$ there exist $\tau\in \R^+_0$ and an input $(s,W,b):[0,\tau]\to \R\times\R^{n\times n}\times \R^n$ for the ensemble control system~\eqref{Product} so that its solution $X(t)$ starting at $X(0)=Y(0)$ satisfies: 
\[
\vert X_{\bullet j}(\tau)-Y_{\bullet j}(1)\vert_{\infty}\le \zeta,
\] 
for $j\in\{1,\hdots,d\}$ which is a restatement of: 
\[
\Vert \phi-f\Vert_{L^\infty(\sampleset)}\le \zeta. 
\] 
In particular, the flow $\phi$ will be defined by the solution $X(t)$.

To simplify the proof we make two claims whose proofs are postponed to after the conclusion of the main argument.

\textbf{Claim 1:} Along the curve $Y$, the ordering of the entries of multiple rows of $Y(t)$ does not change at the same time instant. More precisely, for every $t\in [0,\tau]$, there exists a sufficiently small $\rho\in \R^+$ so that there exists at most one $i\in \{1,\hdots,n\}$ and at most one pair $(j,k)\in\{1,\hdots,d\}^2$ so that $Y_{ij}(t_1)-Y_{ik}(t_1)>0$ for all $t_1\in [t-\rho,t[$ and $Y_{ij}(t_1)-Y_{ik}(t_1)<0$ for all $t_1\in ]t,t+\rho]$. 

\textbf{Claim 2: } The interval $[0,1]$ can be divided into finitely many intervals:
$$]0=t_0,t_1[\,\cup\, ]t_1,t_2[\,\cup\hdots\cup\, ]t_{Q-1},t_Q=1[,$$ 
where $Q$ is a positive integer, so that the ordering of the elements in the rows of $Y$ does not change in these intervals. 

We now proceed with the main argument. We assume that:
\begin{equation}
\label{ZeroMeasureSet}
\prod_{1\le i<j\le d}(A_{\ell i}-A_{\ell j})\ne 0,\quad \forall \ell\in\{1,\hdots,n\},
\end{equation}
where $A$ is the matrix whose columns are the $d$ elements of $\sampleset$. Since the set of points violating~\eqref{ZeroMeasureSet} is a zero measure set, we can always perturb $\sampleset$ to ensure this assumption is satisfied. Note that~\eqref{ZeroMeasureSet} is violated at the time instants $t_1,\hdots,t_{Q-1}$ and possibly also at $t_Q=1$.

Recall that by Claim 2, no changes in the ordering of the entries of the rows of $Y(t)$ occur in the intervals $]t_q,t_{q+1}[$, $q\in \{0,\hdots,Q-1\}$. Hence, we denote by $S_q$ the set of matrices in $\R^{n\times d}$ that have the same ordering as $Y(t)$ in the interval $ ]t_q,t_{q+1}[$. Note that the sequence of visited sets $S_q$ is uniquely determined by $Y(t)$, and hence this dependence is implicit in our chosen notation. Moreover, by~\eqref{ZeroMeasureSet} we have $Y(0)\in S_0$. The control input will be constructed so that the sequence of sets $S_q$ visited by $X(t)$ as $t$ ranges from $0$ to $\tau$ will be the same as the sequence of sets $S_q$ visited by $Y(t)$ as $t$ ranges from $0$ to $1$. However, the time instants at which the switch from $S_q$ to $S_{q+1}$ occurs along the solution $X(t)$ are different from those along the solution $Y(t)$, which are given by $t_q$. The ability to design an input ensuring that a solution of~\eqref{Product} starting at an arbitrary point in $S_q$, for any given $q$, can reach an arbitrary point of $S_q$ is ensured by Proposition~\ref{Prop:Monotone}. Moreover, such input results in a flow that is monotone. Therefore, in the remainder of the proof we only need to establish that the solution of~\eqref{Product} can move from $S_q$ to $S_{q+1}$ along a monotone flow. Once this is established, we can compose the intermediate flows specifying how to select the inputs for the part of the flow that is in $ S_{q} $, as well as the part that corresponds to exiting $ S_{q} $  and entering $S_{q+1}$. This allows us to obtain a monotone flow $\phi$ taking $Y(0)$ to $Y(1)$, if $Y(1)$ belongs to the interior of $S_{Q-1}$.
If $Y(1)$ belongs to the boundary of $S_{Q-1}$, we can design the flow $\phi$ to take $Y(0)$ to any point in the interior of $S_{Q-1}$ and, in particular, to a point that is arbitrarily close to $Y(1)$ since Proposition~\ref{Prop:Monotone} asserts controllability on the interior of $S_{Q-1}.$ This will establish the desired claim that $\Vert \phi-f\Vert_{L^\infty(\sampleset)}\le \zeta$ and any desired $\zeta\in \R^+$. If we then choose $\delta$ and $\zeta$ so as to satisfy $2\omega_f(\delta)+3\zeta\le \varepsilon$, we can invoke Lemma~\ref{Lemma:MonotoneApprox} to conclude the proof. 

It only remains to show that the solution of~\eqref{Product} can move from $S_q$ to $S_{q+1}$ along a monotone flow. There are two situations to consider: $Y_{ij}(t_q-\rho)>Y_{ik}(t_q-\rho)$ changes to $Y_{ij}(t_q+\rho)<Y_{ik}(t_q+\rho)$ or $Y_{ij}(t_q-\rho)<Y_{ik}(t_q-\rho)$ changes to $Y_{ij}(t_q+\rho)>Y_{ik}(t_q+\rho)$, for some $ i, j$, and $ k>j $. 
It is clearly enough to consider one of these cases, and we assume the latter in what follows. In addition to this, from now on, we fix the indices $ i,j$, and $ k $. 

The vectors $Y_{\bullet j}(t_q)$ and $Y_{\bullet k}(t_q)$ cannot satisfy $Y_{\bullet j}(t_q)\preceq Y_{\bullet k}(t_q)$, since monotonicity of the map $h_{t}$ would imply the order is maintained for all time, i.e., $h_t(x^j)=Y_{\bullet j}(t)\preceq Y_{\bullet k}(t)=h_t(x^j)$ for $t\ge t_q$. Since  $Y_{\bullet j}(t_q)\preceq Y_{\bullet k}(t_q)$ does not hold there must exist $r\in \{1,\hdots,n\}$ such that $Y_{ r j}(t_q)>Y_{r k}(t_q)$. We claim the input defined by $s=1$, $b=0$, and $W$ being the matrix whose only non-zero entry is $W_{ir}=1$, can be used to drive a suitably\footnote{Since Proposition~\ref{Prop:Monotone} asserts controllability in the set $S_q$, we are free to choose the state $\Xin$.} chosen state $\Xin\in S_q$ at time $\tin$ to the some state $\Xf\in S_{q+1}$ at time $\tf$. To establish this claim we need to specify the states $\Xin$ and $\Xf$ as well the time instants $\tin$ and $\tf$. First, however, we observe that when using this input, the control system~\eqref{ControlSystemModel} becomes the vector field:
\begin{equation}
\label{SwitchVF}
\sigma(x_r)\frac{\partial}{\partial x_i}.
\end{equation}
Since by our assumption $D\sigma\ge 0$, we conclude by~\eqref{ConditionMonotoneVF} that this vector field is monotone. Moreover, if we integrate the ensemble differential equation defined by the vector field~\eqref{SwitchVF} we obtain:
$$X_{i'j}(\tin+t)=X_{i'j}(\tin),$$
for all 
$i'\in\{1,\hdots,n\}$ with $i'\ne i$, $t\in [0,\tf-\tin] $, and:
\begin{equation}
\label{SolutionDE}
X_{ij}(\tin+t)=X_{ij}(\tin)+t\sigma(X_{rj}(\tin)),
\end{equation}
where $ t\in [0,\tf-\tin] $. 
We now assume, without loss of generality, that $\Xin_{i\bullet}$ is ordered as follows: $\Xin_{i1}<\Xin_{i2}<\hdots<\Xin_{id}$. Recall that $ j $ and $ k>j $ were indices where the order of entries of $ Y_{i\bullet} $ are swapped, at time $ t_q $. We claim that $ k=j+1$; suppose on the contrary that there is an index $ k' $ such that $\Xin_{ij}  < \Xin_{ik'} <\Xin_{ik}$. This would violate the existence of a continuous path from $Y(t_q-\rho)$ to $Y(t_q+\rho)$ for which claim 1 holds.
We already established that there exists $r\in \{1,\hdots,n\}$ such that $Y_{rj}(t_q)>Y_{r(j+1)}(t_q)$.  By continuity of $Y$, we have $Y_{rj}(t_q-\theta)>Y_{r(j+1)}(t_q-\theta)$ for sufficiently small $\theta\in \R^+$. This shows that elements $A\in S_q$ satisfy $A_{rj}>A_{r(j+1)}$. As  $X(\tin)\in S_q$, we also have $X_{rj}(\tin)>X_{r(j+1)}(\tin)$. Moreover, $\sigma$ being an increasing function (recall the assumption $D\sigma\ge 0$) implies $\sigma(X_{rj}(\tin))>\sigma(X_{r(j+1)}(\tin))$. Hence, and for any $t^*_j\in \R^+$ satisfying:
\begin{equation}
\label{TimeBound1}
t^*_{j}>\frac{X_{i(j+1)}(\tin)-X_{ij}(\tin)}{\sigma(X_{rj}(\tin))-\sigma(X_{r(j+1)}(\tin))},
\end{equation}
it follows from~\eqref{SolutionDE} that 
\[
X_{ij}(\tin+t^*_{j})>X_{i(j+1)}(\tin+t^*_{j}).
\]
For any other entries $X_{ij'}(\tin)$ and $X_{i(j'+1)}(\tin)$ with $j'\ne j$, we will have 
\[
X_{ij'}(\tin+t^*_{j'})=X_{i(j'+1)}(\tin+t^{*}_{j'})
\] 
at time:
\begin{equation}
\label{TimeBound2}
t^{*}_{j'}=\frac{X_{i(j'+1)}(\tin)-X_{ij'}(\tin)}{\sigma(X_{rj'}(\tin))-\sigma(X_{r(j'+1)}(\tin))}.
\end{equation}
Noting that $t^{*}_{j'}$ is an increasing function of $X_{i(j'+1)}(\tin)-X_{ij'}(\tin)$, we conclude that if $X_{i(j'+1)}(\tin)-X_{ij'}(\tin)$ is sufficiently large, we have $t^{*}_{j'}>t^*_{j}$. Hence, for any $\tin$, if we choose $\Xin=X(\tin) \in S_q$ such that $\min_{j'\ne j} t_{j'}^*>t^*_{j}$, and choose $\tf=t^*_{j}$ and $\Xf=X(\tin+t^*_{j})$, we have $\Xin \in S_q$ and $\Xf\in S_{q+1}$ as desired. 

\textbf{Proof of Claim 1:} We argue that if the statement does not hold for the chosen set $\sampleset$, it can always be enforced by an arbitrarily small change to the elements of $\sampleset$. Let us fix $i\in\{1,\hdots,n\}$ and $j,k,l\in \{1,\hdots,d\}$, and suppose we want to avoid $Y_{ij}(t)=Y_{ik}(t)=Y_{il}(t)$ for any $t\in [0,\tau]$. The set of initial conditions to be avoided is thus:
\begin{align*}
B
=&\bigcup_{t\in [0,\tau]} \left\{ A\in \R^{n\times d}\,\,\vert\,\, \widetilde{Y}_i(A_{\bullet j},t)\right.\\
=&\left.\widetilde{Y}_i(A_{\bullet k},t)\land \widetilde{Y}_i(A_{\bullet k},t)=\widetilde{Y}_i(A_{\bullet l},t)\right\}.
\end{align*}
Here $\widetilde{Y}_i(A_{\bullet j},t)$ is the $ij$ entry of the curve $\widetilde{Y}(A,t)$ defined by:
$$
\widetilde{Y}(A,t)=\left[h(A_{\bullet 1},t)\vert h(A_{\bullet 2},t)\vert\hdots\vert h(A_{\bullet d},t)\right].
$$
It is convenient to define this set by the image of the smooth map $F:[0,\tau]\times\R^{nd-2}\to \R^{n\times d}$. To define $F$, note that the set:
$$N=\{A\in \R^{n\times d}\,\,\vert\,\, A_{ij}=A_{ik}=A_{il}\},$$ 
is an affine subspace of $\R^{n\times d}$ and thus a submanifold of dimension $nd-2$. Let $W_1,\hdots W_{nd-2}$ be a collection of vector fields on $\R^{n\times d}$ spanning\footnote{A globally defined basis for the tangent space of $N$ exists since $N$ is an affine manifold.} the tangent space to $N$. Using these vector fields, we define the map $F$ as:
$$F(t,r_1,\hdots,r_{nd-2})=\VecE^{-t}\circ W_1^{r_1}\circ \hdots \circ W_{nd-2}^{r_{nd-2}}(0).$$
We can observe that:
$$\bigcup_{(r_1,\hdots,r_{nd-2})\in \R^{nd-2}}W_1^{r_1}\circ \hdots \circ W_{nd-2}^{r_{nd-2}}(0)=N,$$ 
and thus:
$$\bigcup_{(t,r_1,\hdots,r_{nd-2})\in [0,\tau]\times\R^{nd-2}} \VecE^{-t}\circ W_1^{r_1}\circ \hdots \circ W_{nd-2}^{r_{nd-2}}(0)=B.$$ 
Also note that $F$ is a smooth map, as it is a composition of smooth flows. Moreover, its domain is a manifold with boundary of dimension smaller than the dimension of its co-domain. Hence, it follows from Corollary 6.11 in~\cite{SmoothManifolds} that the image of $F$ has zero measure in $\R^{n\times d}$. We can similarly show that all the other ordering changes to be avoided result in zero measure sets. Since there are finitely many of these sets to be avoided, and a finite union of zero measure sets still has zero measure, we conclude that Claim 1 can always be enforced by suitably perturbing the elements of $\sampleset$ if necessary.

\textbf{Proof of Claim 2:} To show this claim is satisfied, let $\gamma_{ijk}:\R\to \R$  be the function defined by $\gamma_{ijk}(t)=Y_{ij}(t)-Y_{ik}(t)$. The instants $t_q\in\{0,1,\hdots,Q\}$, correspond to the zeros of $\gamma_{ijk}$, i.e., $\gamma_{ijk}(t_i)=0$. Since $Y$ is an analytic function, the function $\gamma_{ijk}$ is also analytic and its zeros are isolated. Therefore, the function $\gamma_{ijk}$ restricted to the compact set $[0,\tau]$ only has finitely many zeros. Since there are finitely many functions $\gamma_{ijk}$ as $(i,j,k)$ ranges on $\{1,\hdots,d\}^3$, there are only finitely many instants $t_q$.
\end{proof}

\begin{proof}[Corollary~\ref{theorem:n+1}]
Since the map $f$ is continuous and defined on a compact set, it follows from the Stone-Weierstass theorem that there exists a polynomial $\tilde{f}:E\to \R^n$ satisfying $\Vert f-\tilde{f}\Vert_{L^\infty(E)} \le\frac{\varepsilon}{2}$. We now construct the mapping $h:\R^{2n+1}\times [0,1]\to \R^{2n+1}$, an injection $\alpha:\R^n\to \R^{2n+1}$, and a projection $\beta:\R^{2n+1}\to \R^n$ satisfying:
\begin{eqnarray}
\label{Prop1}
&& \beta\circ h_0\circ \alpha(x)=x\\
\label{Prop2}
&& \beta\circ h_1\circ \alpha(x)=\tilde{f}(x)\\
\label{Prop3}
&& \forall \tau\in[0,1], t\in[\tau,1]\\
&&h_\tau \circ \alpha(x)\preceq h_\tau \circ\alpha(x')\implies h_{t} \circ \alpha(x)\preceq h_{t} \circ\alpha(x').\notag
\end{eqnarray}
The map $h_t$ is given by: 
\begin{equation}
\label{Homotopy}
h_t(x,y,z)=\left((1-t)x+t(\tilde{f}(x)+\kappa \mathbf{1}\mathbf{1}^Tx),\textcolor{blue}{y},\textcolor{blue}{t}z\right),
\end{equation}
where $(x,y,z,t)\in \R^n\times \R^n\times \R\times \R$, $\mathbf{1}\in \mathbb{R}^n$ is the vector all of whose entries are $1$, and: 
\[
\kappa=\max_{\stackrel{i,j\in \{1,\ldots,n\}}{x\in E}}\left\vert\frac{\partial \tilde{f}_i}{\partial x_j}\right\vert,
\]
which is well-defined as we have taken $ \tilde{f} $ to be a polynomial and $E$ is compact. The injection  $\alpha$ and the projection $\beta$ are given by:
$$\alpha(x)=(x,x,\mathbf{1}^T x),\qquad \beta(x,y,z)=(x-\kappa \mathbf{1} z).$$
One can easily check that properties~\eqref{Prop1} and~\eqref{Prop2} hold. For property~\eqref{Prop3}, we first note that monotonicity of the map $x\mapsto \tilde{f}(x)+\kappa \mathbf{1}\mathbf{1}^Tx$ follows from the definition of $\kappa$, and criterion~\eqref{ConditionMonotoneMap}. Using criterion~\eqref{ConditionMonotoneMap} again we see that the map $x\mapsto (1-t)x+t(\tilde{f}(x)+\kappa \mathbf{1}\mathbf{1}^Tx)$ is also monotone for each $t\in [0,1]$. Finally, we conclude that $h_t$ is monotone for every $t\in [0,1]$. Assume now that $h_\tau \circ \alpha(x)\preceq h_\tau \circ\alpha(x')$ for some $\tau\in [0,1]$. In particular this implies that $x\preceq x'$ by definition of $h_t$. It then follows from monotonicity of $h_t$ and $\alpha$ that we have $h_t\circ\alpha(x)\preceq h_t\circ \alpha(x')$ for $t\in[\tau,1]$.

Consider now the function $Y$ defined by equality~\eqref{FunctionY} in the Proof of Theorem~\ref{Thm:MonotoneFlow}. Using such definition for $Y$, but with the function $h_t$ defined in~\eqref{Homotopy}, we conclude from the proof of Theorem~\ref{Thm:MonotoneFlow} the existence of an input $(s,W,b):[0,\tau]\to \R\times\R^{(2n+1)\times (2n+1)}\times \R^{2n+1}$ so that the flow $\phi^\tau:\R^{2n+1}\to \R^{2n+1}$ of~\eqref{ControlSystemModel} satisfies:
\begin{equation}
\label{BoundFlows}
\Vert h_1\circ\alpha-\phi^\tau\circ \alpha\Vert_{L^\infty(E)}\le \frac{\varepsilon}{2(1+\kappa)}.
\end{equation}
We therefore have:
\begin{eqnarray}
\left\Vert \tilde{f}-\beta\circ \phi^\tau\circ\alpha\right\Vert_{L^\infty(E)}&=&\left\Vert \beta\circ h_1\circ \alpha-\beta\circ \phi^\tau\circ\alpha\right\Vert_{L^\infty(E)}\notag\\
&\le & (1+\kappa)\left\Vert h_1\circ \alpha- \phi^\tau\circ\alpha\right\Vert_{L^\infty(E)}\notag\\
&\le& \frac{\varepsilon}{2}\notag ,
\end{eqnarray}
where the first inequality follows from $(1+\kappa)$ being the Lipschitz constant of $\beta$ and the second from~\eqref{BoundFlows}. Finally, we use the preceding inequality to establish:
\begin{align*}
\left\Vert f-\beta\circ \phi^\tau\circ\alpha\right\Vert_{L^\infty(E)}\le& \left\Vert f-\tilde{f}\right\Vert_{L^\infty(E)}\\
&+\left\Vert \tilde{f}-\beta\circ \phi^\tau\circ\alpha\right\Vert_{L^\infty(E)}\\
\le& \frac{\varepsilon}{2}+\frac{\varepsilon}{2}=\varepsilon.
\end{align*}
\end{proof}

\section{Outlook}\label{section:outlook}
We have provided sharp uniform approximation results for residual neural networks under mild assumptions on the class of activations functions. Our approach combines, for the first time, control-theoretic ideas to obtain such results. In particular, we utilize Lie algebraic techniques to study controllability of ensemble control systems and elucidate the memorization power of neural networks. We additional leverage monotonicity to build upon controllability and establish uniform approximation. Although the results in this paper are existential in nature, recent results suggest that training algorithms can be adapted to enforce monotonicity and thus guarantee uniform approximation bounds~\cite{MM-BG-PT:21}. The focal objective of our future work is the use of our uniform approximation results to place neural networks in control loops, while providing safety and robustness guarantees.

\section{Appendix}

The proof of Theorem~\ref{Thm:Controllability} is based on two  technical results. The first characterizes the rank of a certain matrix that will be required for our controllability result. In essence, the proof of this result follows from~\cite[Proposition~1]{krattenthaler2001advanced}, however, we provide a proof for completeness.
Throughout the proof, for simplicity of presentation, we assume that $ \sigma $ is injective and satisfies the quadratic differential equation $D\sigma=a_0+a_1\sigma+a_2\sigma^2$ with $a_2\ne 0$. Nevertheless, as indicated in Definition~\ref{def:activation-class}, it is enough for $ D^j\sigma$ to be injective and to satisfy the mentioned quadratic differential equation for some $ j \in \mathbb{N}_0$.

\begin{lemma}\longthmtitle{A determinant formula for quadratic differential equations}
\label{LemmaDet}
Let $\xi:\R\to \R$ be a function that satisfies the quadratic differential equation: 
\[
D\xi(x)=a_0+a_1\xi(x)+a_2\xi^2(x),
\]
where $a_0,a_1,a_2\in \R$. Suppose that derivatives of $ \xi $ of up to order $ (\ell -2) $ exist at $ \ell $ points $ x_1,\ldots, x_{\ell}\in \R $. 
Then, the determinant of the matrix:
\begin{align}
&\label{VandermontLikeMatrix}
L(x_1,x_2,\ldots, x_\ell)\cr
&=\begin{bmatrix}
1 & 1 & \hdots & 1\\
\xi(x_1) & \xi(x_2) & \hdots & \xi(x_\ell)\\
D \xi(x_1) & D\xi(x_2) & \hdots & D\xi(x_\ell)\\
\vdots&\vdots&\ddots&\vdots\\
D^{\ell-2} \xi(x_1) & D^{\ell-2}\xi(x_2) & \hdots &D^{\ell-2}\xi(x_\ell)
\end{bmatrix},\qquad\qquad
\end{align}
is given by:
\begin{equation}
\label{VandermontLikeDeterminant}
\det L(x_1,x_2,\ldots, x_\ell)=\prod_{i=1}^{\ell-2} i!a_2^{i} \prod_{1\le i<j\le\ell}(\xi(x_i)-\xi(x_j)).
\end{equation}
\end{lemma}

\begin{proof}
We assume that the elements of the set $ \{x_1,x_2,\ldots, x_\ell\} $ are distinct, as otherwise, the determinant is clearly zero. We also assume that $ \ell\geq 3 $ to exclude the trivial case. 
First, by the Vandermonde determinant formula, we have that:
\begin{align}
V_0(x_1,x_2,\ldots, x_\ell):=&
\begin{vmatrix}
1 & 1 & \hdots & 1\\
\xi(x_1) & \xi(x_2) & \hdots & \xi(x_\ell)\\
\xi^2(x_1) & \xi^2(x_2) & \hdots & \xi^2(x_\ell)\\
\vdots&\vdots&\ddots&\vdots\\
\xi^{\ell-1} (x_1) & \xi^{\ell-1} (x_2) & \hdots &\xi^{\ell-1} (x_\ell)
\end{vmatrix}\cr
=&\prod_{1\le i<j\le\ell}(\xi(x_i)-\xi(x_j)).\label{eq:Van}
\end{align}
Our proof technique is to use elementary row operations to construct the determinant of $ L(x_1,x_2,\ldots, x_\ell) $ from~\eqref{eq:Van}. To illustrate the idea, let us use~\eqref{eq:Van} to show that:
\begin{align*}
V_1(x_1,x_2,\ldots, x_\ell):=&\begin{vmatrix}
1 & 1 & \hdots & 1\\
\xi(x_1) & \xi(x_2) & \hdots & \xi(x_\ell)\\
D\xi(x_1) & D\xi(x_2) & \hdots & D\xi(x_\ell)\\
\xi^3(x_1) & \xi^3(x_2) & \hdots & \xi^3(x_\ell)\\
\vdots&\vdots&\ddots&\vdots\\
\xi^{\ell-1} (x_1) & \xi^{\ell-1} (x_2) & \hdots &\xi^{\ell-1} (x_\ell)
\end{vmatrix}\\
=&a_2\prod_{1\le i<j\le\ell}(\xi(x_i)-\xi(x_j)).
\end{align*}
For later use, we denote by $ V_{i} (x_1,x_2,\ldots, x_\ell) $ the determinant of the matrix constructed by substituting rows $ 3 $ to $ i $ in $ V_{0} (x_1,x_2,\ldots, x_\ell) $ by derivatives of order $ 1 $ to $ i-2 $, respectively.  First, note that multiplying the third row of $L(x_1,x_2,\ldots, x_\ell)$ by $a_2$ leads to:
\begin{align*}
\begin{vmatrix}
1 & 1 & \hdots & 1\\
\xi(x_1) & \xi(x_2) & \hdots & \xi(x_\ell)\\
a_2\xi^2(x_1) & a_2\xi^2(x_2) & \hdots & a_2\xi^2(x_\ell)\\
\xi^3(x_1) & \xi^3(x_2) & \hdots & \xi^3(x_\ell)\\
\vdots&\vdots&\ddots&\vdots\\
\xi^{\ell-1} (x_1) & \xi^{\ell-2} (x_1) & \hdots &\xi^{\ell-1} (x_\ell)
\end{vmatrix}=a_2V_0(x_1,x_2,\ldots, x_\ell).
\end{align*}
Moreover, by the fact that the determinant is unchanged by adding a constant multiple of a row to another row, 
we can conclude that:
\[
V_1(x_1,x_2,\ldots, x_\ell)=a_2V_0(x_1,x_2,\ldots, x_\ell),
\]
proving the claim. The idea of the proof is to use this same procedure, row by row, to construct $ D^i\xi(x_j) $ in the entry $ (i+2)\times j $ of the matrix. In order to proceed, however, we need to find a formula for $ D^i\xi(x) $, where $ x \in \R $. Note that, for $ i\geq 2 $, we have that:
\begin{align*}
D^i\xi(x)
=&a_1D^{i-1}\xi(x)+2a_2\frac{d}{dx^{i-2}}(\xi(x)D\xi(x))\\
=&a_1D^{i-1}\xi(x)\\
&+2a_2\sum_{k=0}^{i-2}\binom{i-1}{k}D^{i-k-2}\xi(x)D^{k+1}\xi(x),
\end{align*}
and $ D^i\xi(x)$, as a polynomial in $ \xi(x) $, is of degree $ (i+1) $. We now make an observation that finishes the proof. In particular, in the computation of $ V_1(x_1,x_2,\ldots, x_\ell) $ and in order to construct $ D\xi(x) $ in the third row, we only needed to know the coefficient of the highest degree monomial, in terms of $ \xi(x) $, that constitutes $ D\xi(x) $. In other words, the lower degree terms do not contribute to the determinant, as they can be constructed, without changing the determinant, 
from previous rows. Using this observation, the term 
$ a_1D^{i-1}\xi(x) $ in the expansion of $ D^i\xi(x) $ does not contributed to $ V_{i} (x_1,x_2,\ldots, x_\ell) $, as it can be added from the previously constructed rows. Using this reasoning for all $ i $, we conclude that the determinant of $ L(x_1,\ldots, x_{\ell}) $ is independent of $ a_1 $, and $ a_0 $. Substituting $ a_0=0 $ and $ a_1=0 $, since $ D^i\xi(x)=i!a_2^i\xi^{i+1} $, we have that:
\begin{align*}
\det &L(x_1,\ldots, x_{\ell})\\
&=\prod_{i=1}^{\ell-2} i!a_2^{i} V_0(x_1,x_2,\ldots, x_\ell)\\
&=\prod_{i=1}^{\ell-2} i!a_2^{i}  \prod_{1\le i<j\le\ell}(\xi(x_i)-\xi(x_j)),
\end{align*}
as claimed. 
\end{proof}

Our second technical result is stated next, for which we provide an elementary proof to keep the manuscript self-contained.

\begin{proposition}
\longthmtitle{An open dense submanifold of $\R^{n\times d}$}
\label{Prop:Connected}
Let $N\subset \R^{n\times d}$ be the set defined by:
\begin{equation}
N=\{A\in  \R^{n\times d} \ \vert \ \prod_{1\le i<j\le d}(A_{\ell i}-A_{\ell j})=0, \ \ell \in\{1,\ldots, n\}\}.\notag
\end{equation}
The set $M=\R^{n\times d}\backslash N$ is an open and dense submanifold of $\R^{n\times d}$ which is connected when $n>1$.
\end{proposition}
\begin{proof}
Note that $N$ is a finite union of vector subspaces of $\R^{n\times d}$, hence topologically closed. Therefore, $\R^{n\times d}\backslash N$ is an open and dense subset of $\R^{n\times d}$ and thus a submanifold of dimension $n d$. It remains to show that $M$ is connected.

Let $\Ain,\Af \in M$, and assume that $ n>1 $. We prove that there exists a continuous curve $ \gamma:[0,n]\to M $ connecting $\Ain$ to $\Af$, i.e., $\gamma(0)=\Ain$ and $\gamma(n)=\Af$. Since $\Ain\in M$ there exists $\lin\in \{1,\hdots,n\}$ so that 
\[ 
\prod_{1\le i<j\le d}(A_{\lin i}-A_{\lin j})\neq 0.
\] 
Similarly, since $\Af\in M$ there exists $\lf\in \{1,\hdots,n\}$ so that 
\[
\prod_{1\le i<j\le d}(A_{\lf i}-A_{\lf j})\neq 0.
\] 
We first consider the case where $\lin\ne \lf$ (which is possible since $n>1$). Without loss of generality assume that $\lin=n$ and $\lf=1$ and let $ \gamma^{k}: \R^{n\times d}\times [k-1,k]\rightarrow \R^{n\times d} $ be defined as:
\[
\gamma^{k}_\lambda(A)=\gamma^k(A,\lambda)=\begin{bmatrix}
A_{1 \bullet}\\
\vdots\\
A_{k-1 \bullet}\\
A_{k\bullet}+(\lambda-(k-1))\Af_{k\bullet}\\
A_{k+1\bullet}\\
\vdots\\
A_{n\bullet}
\end{bmatrix},
\]
where $ k\in \{1,\hdots,n\} $ and $A_{k\bullet}$ denotes the $k$th row of $A$. We now define the curve $ \gamma:  [0,n]\rightarrow \R^{n\times d} $ by:
$$\gamma(\lambda)=\gamma^k_{\lambda}\circ \gamma^{k-1}_{k-1}\circ\hdots\circ \gamma^2_2\circ \gamma^1_1(\Ain),\quad \lambda\in [k-1,k],$$
and note that $ \gamma(\lambda) \in M $ for all $ \lambda \in [0,n] $. This is because, by definition, there exists at least one index $ \ell \in \{1,\ldots, n\} $ such that 
\[ 
\prod_{1\le i<j\le d}(\gamma_{\ell i}(\lambda)-\gamma_{\ell j}(\lambda))\neq 0.
\] 
When $\lambda\le n-1$, we can choose $\ell$ to be $\lin$ because $\gamma_{\lin \bullet}(\lambda)=\gamma_{n \bullet}(\lambda)=\Ain_{n\bullet}$. When $\lambda\ge n-1$, we can choose $\ell$ to be $\lf$ because $\gamma_{\lf \bullet}(\lambda)=\gamma_{1 \bullet}(\lambda)=\Af_{1\bullet}$. Since $\gamma$ is the composition of continuous functions, it is continuous. Moreover, by construction, $ \gamma(0) =\Ain$ and $ \gamma(n)=\Af$.

We now consider the case where $\lin=\lf$. Since $n>1$, we can choose $A\in M$ so that 
\[ 
\prod_{1\le i<j\le d}(A_{\ell i}-A_{\ell j})\neq 0 
\] 
with $\ell\ne \lf$ and $\ell\ne \lin$. By the previous argument, there is a continuous curve connecting $\Ain$ to $A$ without leaving $M$ and there is also a continuous curve connecting $A$ to $\Af$ without leaving $M$. Therefore, their concatenation produces the desired continuous curve $\gamma$ connecting $\Ain$ to $\Af$ and the proof is finished.
\end{proof}

	\bibliographystyle{plain}  
\bibliography{alias,bib}

\end{document}